\documentclass[letterpaper, 10 pt, conference]{ieeeconf}  

\IEEEoverridecommandlockouts                              

\usepackage{graphics} 
\usepackage{epsfig} 
\usepackage{times} 
\usepackage{amsmath} 
\usepackage{amssymb}  
\usepackage{amsfonts}

\usepackage{amsthm}
\usepackage{mathtools}

\usepackage{siunitx}
\usepackage{comment}

\usepackage[bookmarks=false]{hyperref}  
\usepackage{url}            
\usepackage{booktabs}       
\usepackage{nicefrac}       
\usepackage{microtype}      

\DeclareMathOperator*{\argmax}{arg\,max}

\newtheorem{lemma}{Lemma}
\newtheorem{definition}{Definition}
\newtheorem{theorem}{Theorem}
\newtheorem{proposition}{Proposition}
\newtheorem{corollary}{Corollary}
\newtheorem{assumption}{Assumption}

\allowdisplaybreaks

\title{\LARGE \bf
Discrete-Time Mean Field Control with Environment States
}

\author{Kai Cui, Anam Tahir, Mark Sinzger and Heinz Koeppl
\thanks{The authors are with the Department of Electrical Engineering, Technische Universität Darmstadt, 64287 Darmstadt, Germany. Contact: {\tt\small  \{kai.cui, anam.tahir, mark.sinzger, heinz.koeppl\}@bcs.tu-darmstadt.de}}%
}

\begin{document}

\maketitle
\thispagestyle{empty}
\pagestyle{empty}

\begin{abstract}
Multi-agent reinforcement learning methods have shown remarkable potential in solving complex multi-agent problems but mostly lack theoretical guarantees. Recently, mean field control and mean field games have been established as a tractable solution for large-scale multi-agent problems with many agents. In this work, driven by a motivating scheduling problem, we consider a discrete-time mean field control model with common environment states. We rigorously establish approximate optimality as the number of agents grows in the finite agent case and find that a dynamic programming principle holds, resulting in the existence of an optimal stationary policy. As exact solutions are difficult in general due to the resulting continuous action space of the limiting mean field Markov decision process, we apply established deep reinforcement learning methods to solve the associated mean field control problem. The performance of the learned mean field control policy is compared to typical multi-agent reinforcement learning approaches and is found to converge to the mean field performance for sufficiently many agents, verifying the obtained theoretical results and reaching competitive solutions. 
\end{abstract}


\section{INTRODUCTION}
Reinforcement Learning (RL) has proven to be a very successful approach for solving sequential decision-making problems \cite{sutton2018reinforcement}. Today it has numerous applications e.g. in robotics \cite{kober2013reinforcement}, strategic games \cite{brown2019superhuman} or communication networks \cite{luong2019applications}. Many such applications are modelled as special cases of Markov games, which has led to empirical success in the multi-agent RL (MARL) domain.

However, MARL problems quickly become intractable for large numbers of agents and proposed solutions offer few rigorous guarantees \cite{zhang2021multi}. An increasingly popular approach in resolving this curse of dimensionality are mean field approximation. The main idea is to convert a many-agent system with $N$ indistinguishable and interchangeable agents into a problem where one representative agent interacts with e.g. the empirical state distribution -- the mean field -- of the other agents. Since the $N$-agent model is reduced to a single agent and a mean field, this lends the problem tractability with theoretical guarantees for sufficiently large $N$.

The framework of mean field games (MFG) was first introduced in \cite{huang2006large} and \cite{lasry2007mean} for stochastic differential games and has since been extended to discrete-time \cite{gomes2010discrete, saldi2018markov}. It provides a framework for analyzing many-agent competitive problems, for which learning-based solutions have become increasingly popular \cite{mguni2018decentralised, guo2019learning, cui2021approximately}. Mean field theory applied to the cooperative setting is known as mean field control (MFC), where one assumes that many agents cooperate to achieve Pareto optima \cite{andersson2011maximum, bensoussan2013mean}. MFC has various applications e.g. in smart heating \cite{kizilkale2014collective} or portfolio management \cite{djehiche2016risk}.

The dimensions of the MFC problem are independent of the specific number of agents, making it more tractable. However, solving the MFC problem has the challenge of time-inconsistency due to the non-Markovian nature of the problem \cite{andersson2011maximum, djehiche2015stochastic, djete2019mckean}. A recent way of handling this inherent time-inconsistency problem is to use an enlarged state-action space \cite{pham2018bellman, motte2019mean, gu2019dynamic, gu2020q}. We similarly apply this technique by lifting up the state-action space into its probability measure space, since it will enable usage of dynamic programming and established reinforcement learning methods. 

\begin{figure}
    \centering
    \includegraphics[width=0.7\linewidth]{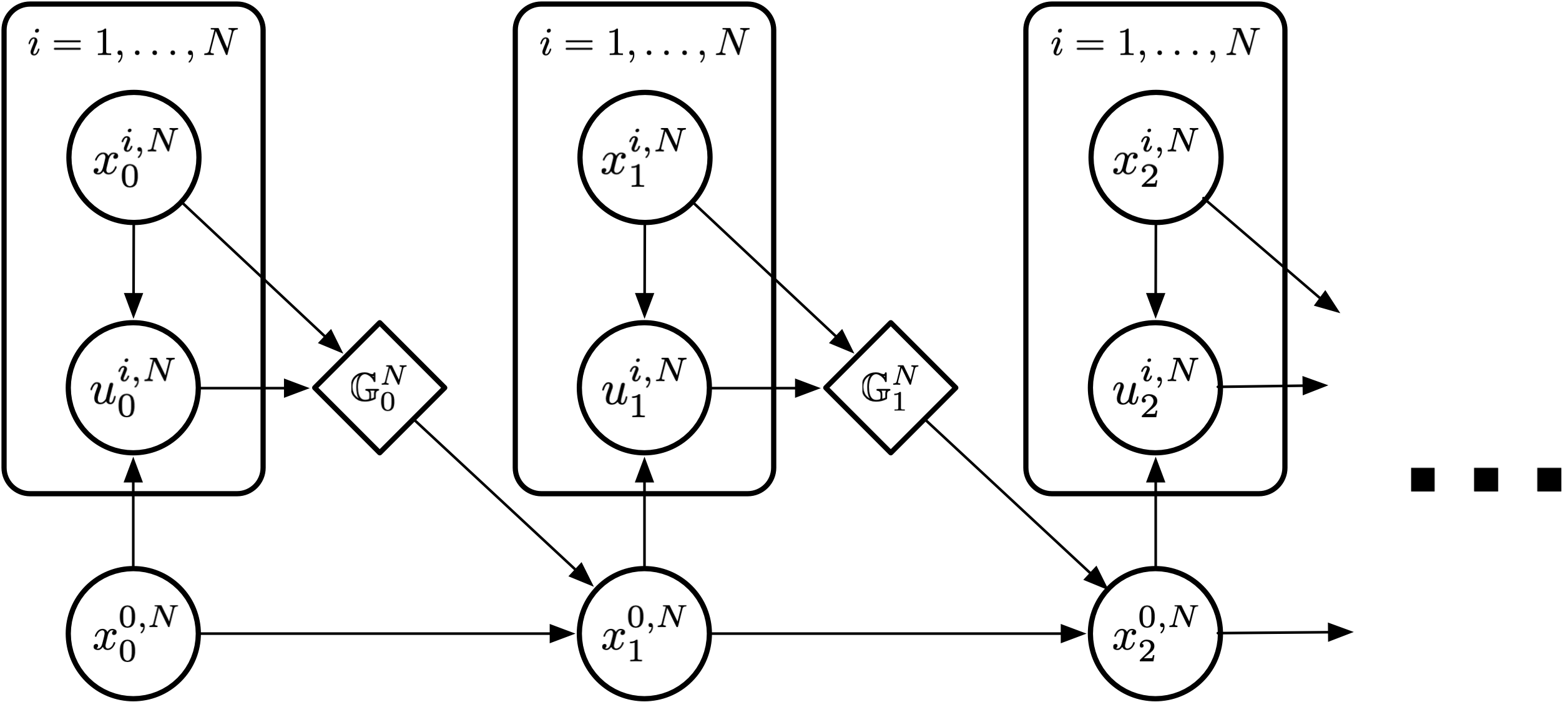}
    \caption{Overview of the multi-agent system as a probabilistic graphical model using plate notation \cite{murphy2012machine}, where circles and diamonds indicate stochastic and deterministic nodes respectively. Each agent $i$ chooses an action $u_t^{i,N}$ conditional on the environment state $x_t^{0,N}$ and local agent state $x_t^{i,N}$, influencing the next environment state $x_{t+1}^{0,N}$ only via their empirical distribution $\mathbb G_t^N$. Agent states are assumed i.i.d. for simplicity of analysis.}
    \label{fig:overview}
\end{figure}

In this work we extend the theory of discrete-time MFC by considering additional environment states. An advantage of discrete-time models is applicability of a plethora of reinforcement learning solutions. Our model can be considered a special case of the MFC equivalent of major-minor mean field games \cite{nourian2013mm, caines2016mm} with trivial major agent policy, which to the best of our knowledge has not been formulated yet. We expect that our results can be generalized, similar to approaches e.g. in \cite{saldi2018markov} for the competitive mean field game, although for deterministic mean fields.

The main contributions of this paper are: (i) We propose a new discrete-time MFC formulation that transforms large-scale multi-agent control problems with common environment states into a simple Markov decision process (MDP) with lifted state-action space; (ii) we rigorously show approximate optimality for sufficiently large systems as well as existence of an optimal stationary policy through a dynamic programming principle, and (iii) associated with this standard discrete-time MDP with continuous action space, we verify our theoretical findings empirically using modern reinforcement learning techniques. As a result, we outperform existing baselines for the many-agent case and obtain a methodology to solve large multi-agent control problems such as the following.

\section{SCHEDULING SCENARIO}
While the concept of mean field limits has been used in queuing systems before, it has mostly been used for the state of the buffer fillings of queues or the number of servers/queues \cite{jsq_mf, khudabukhsh2020generalized}. In this work we use mean fields to represent the state of a large amount of schedulers while modeling the queues exactly. See also Figure~\ref{fig:queue_system} for a visualization of the problem. Note that in principle, our model could be used for any similar resource allocation problem such as allocation of many firefighters to houses on fire.

\begin{figure}
    \centering
    \includegraphics[width=0.7\linewidth]{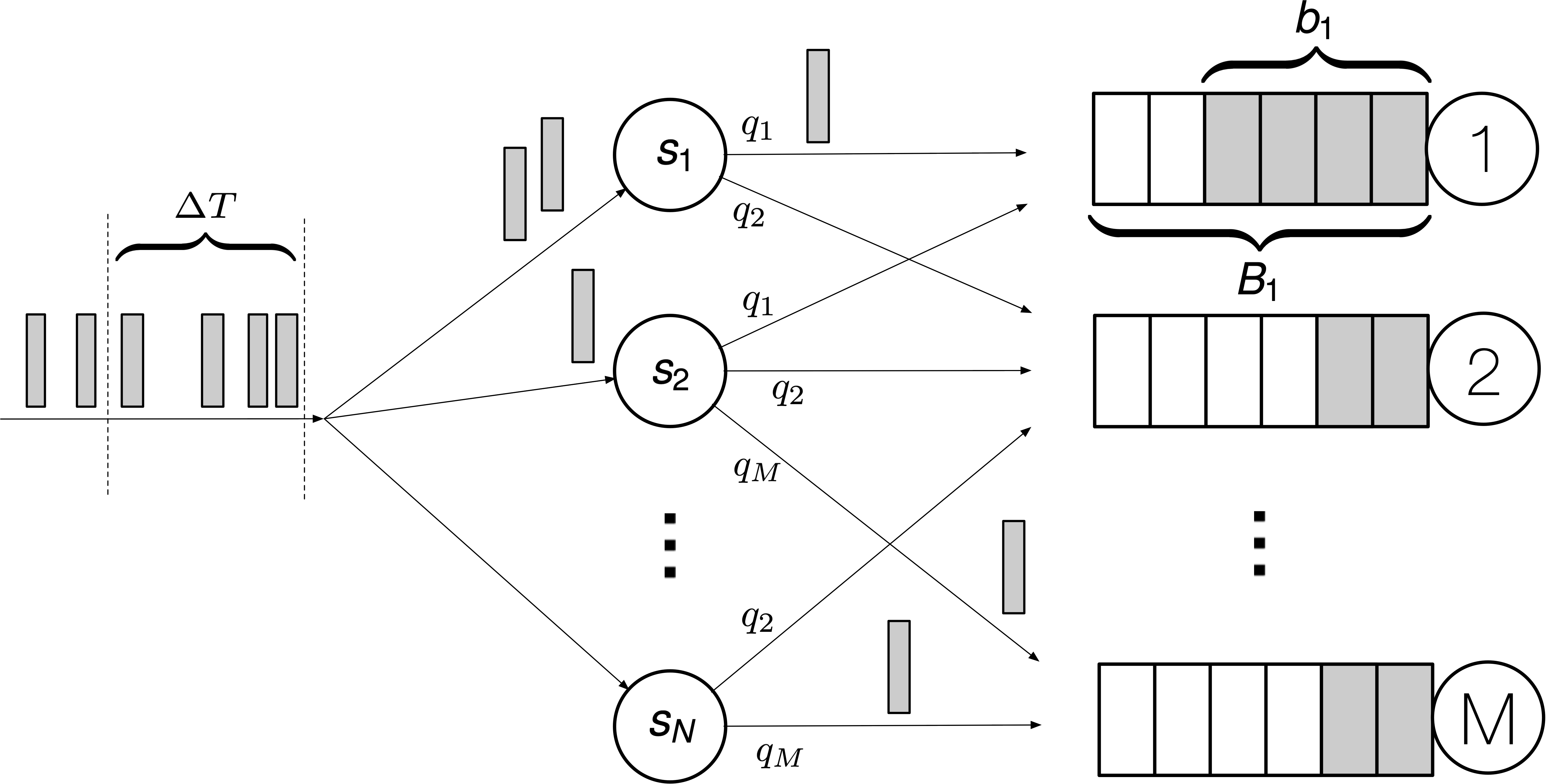}
    \caption{Overview of the queuing system. Many schedulers (middle) obtain packets at a fixed rate (left) that must be assigned to one of the accessible queues (right) such that total packet drops are minimized.}
    \label{fig:queue_system}
\end{figure}

Consider a queuing system with $N$ agents called schedulers, $[s_1,\ldots,s_N]$, and $M$ parallel servers, each with its own finite FIFO queue. Denote the queue filling by $b_i \in \{0,\ldots,B_i\}$, $i = 1,\ldots,M$ where $B_i$ is the maximum buffer space for the $i$-th queue. At any time step $t$, the state $x_t^{i,N} \in \mathcal X$ of a scheduler is the set of queues it has access to. The agent state space $\mathcal X$ therefore consists of all combinations of queue access where every agent has access to at least one of the queues. The environment state is the current buffer filling $x^0 = [b_1, \ldots, b_M]$, where $b_j$ is the buffer filling of queue $j$.

In discrete-time, the number of job arrivals to be assigned at each time step $t$ is Poisson distributed with rate $\lambda \Delta T$ and the number of serviced jobs for each server is Poisson distributed with rate $\beta \Delta T$, where $\Delta T > 0$ can be considered the time span between each synchronization of schedulers. As an approximation, we assume that all queue departures in a time slot happen before the new arrivals, and newly arrived jobs thus cannot be serviced in the same time slot.

We split the total number of job packets which arrive in some time step $\Delta T$ uniformly at random amongst the schedulers. The jobs assigned to each scheduler need to be sent out immediately. Each scheduler decides which of the accessible queues it sends its arrived jobs to during each time step. If a job is mapped to a full buffer, it is lost and a penalty $c_d$ is incurred. The goal of the system is therefore to minimize the number of job drops. At each step of the decision making, we assume that the state of the environment $x^0$ and their own accessible queues are known to the schedulers. 

We can model the dynamics of the environment state dependent on the empirical state-action distribution of all schedulers: Consider agents choosing some choice of queues as their action, where inaccessible queues are treated as randomly picking a destination. In that case, to assign a packet to its destination queue, it is clearly sufficient to consider the empirical distribution: Sampling from the empirical distribution, using the sampled action and, if inaccessible, resampling an accessible queue provides the desired behavior.

\section{MEAN FIELD CONTROL}
In this section, we formulate a $N$-agent model that in the limit of $N \to \infty$ results in a more tractable MFC problem. Importantly, we will then show approximate optimality and a dynamic programming principle for the MFC problem, allowing for application of reinforcement learning.

\textbf{Notation.} \textit{Let $\mathcal A$ be a finite set. We equip $\mathcal A$ with the discrete metric and denote the set of real-valued functions on $\mathcal A$ by $\mathbb R^{\mathcal A}$, For $f \in \mathbb R^{\mathcal A}$ let $\lVert f \rVert_\infty = \max_{a \in \mathcal A} f(a)$. Denote by $|\mathcal A|$ the cardinality of $\mathcal A$. Denote by $\mathcal P(\mathcal A) = \{ p \in \mathbb R^{\mathcal A} \colon p(a) \geq 0, \sum_{a \in \mathcal A} p(a) = 1 \}$ the space of probability simplices, equivalent to the probability measures on $\mathcal A$. Equip $\mathcal P(\mathcal A)$ with the $l_1$-norm $\lVert \mu - \nu \rVert_1 = \sum_{a \in \mathcal A} \left| \mu(a) - \nu(a) \right|$. For readability, we uncurry occurrences of multiple parentheses, e.g. $\pi_t(x^0_t)(x) \equiv \pi_t(x^0_t, x)$. Define $\mu(f) \coloneqq \sum_{a \in \mathcal A} f(a) \mu(a)$ for any $\mu \in \mathcal P(\mathcal A)$, $f \colon \mathcal A \to \mathbb R$.}

\subsection{Finite Agent Model}
Let $\mathcal X$, $\mathcal U$ be a finite state and action space respectively. Let $\mathcal X^0$ be a finite environment state space. For any $N \in \mathbb N$, at each time $t = 0, 1, \ldots$, the states and actions of agent $i = 1, \ldots, N$ are random variables denoted by $x^{i,N}_t \in \mathcal X$ and $u^{i,N}_t \in \mathcal U$. Analogously, the environment state is a random variable denoted by $x^{0,N}_t \in \mathcal X^0$. Define the empirical state-action distribution $\mathbb G_t^N = \frac{1}{N} \sum_{i=1}^N \delta_{(x_t^{i,N}, u^{i,N}_t)} \in \mathcal P(\mathcal X \times \mathcal U)$. For each agent $i$, we consider locally Markovian policies $\pi^i = \{ \pi_t^i \}_{t \geq 0} \in \Pi_N$ from the space of admissible Markov policies $\Pi_N$ where $\pi_t^i \colon \mathcal X^0 \times \mathcal X \to \mathcal P(\mathcal U)$. Further, we define the policy profile $\boldsymbol \pi = (\pi^1, \ldots, \pi^N) \in \Pi_N^N$. 

Acting only on local and environment information may seem like a strong restriction. However, other agent states are uninformative under continuity assumptions as $N \to \infty$ as the interaction between agents will be restricted to the increasingly deterministic empirical state-action distribution.

Let $\mu_0 \in \mathcal P(\mathcal X)$ be the initial agent state distribution, $\mu^0_0 \in \mathcal P(\mathcal X^0)$ the initial environment state distribution and $P^0 \colon \mathcal X^0 \times \mathcal P(\mathcal X \times \mathcal U) \to \mathcal P(\mathcal X^0)$ a transition kernel. The random variables shall follow $x^{0,N}_0 \sim \mu^0_0$ and subsequently
\begin{align}
    x^{i,N}_t &\sim \mu_0, \\
    u^{i,N}_t &\sim \pi_t^i(x^{0,N}_t, x^{i,N}_t), \label{eq:finiteu}\\
    x^{0,N}_{t+1} &\sim P^0(x^{0,N}_t, \mathbb G_t^N),
\end{align}
where for simplicity of further analysis the agent states are always sampled according to $\mu_0$. 

\textbf{Remark.} \textit{While this is a strong dynamics assumption, our formulation is nonetheless sufficient for the scheduling problem. In principle, any results should similarly hold under appropriate assumptions for nontrivial agent state dynamics by considering mean field and environment state together. As this will significantly complicate analysis, an according extension of theoretical results is left to future works.}

Let us introduce another notation. First, define the space of decision rules $\mathcal H \coloneqq \{h \colon \mathcal X \to \mathcal P(\mathcal U)\}$. Then a one-step policy profile $\boldsymbol h = (h^1, \dots, h^N)\in \mathcal H^N$ is an $N$-fold decision rule. Our major example of a one-step policy profile is $(\pi_t^1(x^0), \dots,  \pi_t^N(x^0))$ for fixed $t \geq 0$, fixed $x^0\in \mathcal X^0$ and potentially different policies for the $N$ agents. For given agent state distribution $\mu_0$ and a one-step policy profile $\boldsymbol h \in \mathcal H^N$ let $x^{i,N}\sim \mu_0, u^{i,N} \sim h^i(x^{i,N})$, s.t. $(x^{i,N}, u^{i,N})_{i = 1, \dots, N}$ are independent. Then, consider a random measure $\mathbb G^N_{\boldsymbol h} \in \mathcal P(\mathcal X \times \mathcal U)$ or equivalently its random probability mass function $\mathbb G^N_{\boldsymbol h}\colon \mathcal X \times \mathcal U \to [0,1],$ 
\begin{align}
    \mathbb G^N_{\boldsymbol h}(x,u) \coloneqq \frac 1 N \sum_{i = 1}^N \mathbf{1}_{x,u}(x^{i,N}, u^{i,N}) \, .
\end{align}
Define $\mathcal G^N(\mu_0, \boldsymbol h)$ as the distribution of $\mathbb G^N_{\boldsymbol h}$, so $\mathcal G^N(\mu_0, \boldsymbol h)$ is a distribution over the set $\mathcal P(\mathcal X \times \mathcal U)$ and $\mathbb G^N_{\boldsymbol h} \sim \mathcal G^N(\mu_0, \boldsymbol h)$. Consider the primary example $\boldsymbol h = (\pi_t^1(x^0), \dots,  \pi_t^N(x^0))$. In contrast to the empirical distribution $\mathbb G_t^N$ that depends on a random $x_t^{0,N}$, the random probability mass function $\mathbb G^N_{\boldsymbol h}$ has $x_t^{0,N}= x^0$ fixed. By $\mathbb{E}[\mathbb G^N_{\boldsymbol h}]$ we denote in the following the entry-wise expectation $\{\mathbb{E}[\mathbb G^N_{\boldsymbol h}(x,u)]\}_{(x,u)\in \mathcal X \times \mathcal U}$.

Let $\gamma \in (0,1)$ be the discount factor and $r \colon \mathcal X^0 \times \mathcal P(\mathcal X \times \mathcal U) \to \mathbb R$ a reward function. The goal is to maximise the discounted accumulated reward
\begin{align}
    J^N(\boldsymbol \pi) = \mathbb E \left[ \sum_{t=0}^{\infty} \gamma^t r(x^{0,N}_t, \mathbb G^N_t) \right]
\end{align}
which generalizes optimizing an average per-agent reward 
\begin{align}
    J^N(\boldsymbol \pi) = \sum_{i=1}^{N} \mathbb E \left[ \sum_{t=0}^{\infty} \gamma^t \tilde r(x^{i,N}_t, x^{0,N}_t, \mathbb G^N_t) \right]
\end{align}
for some shared $\tilde r \colon \mathcal X \times \mathcal X^0 \times \mathcal P(\mathcal X \times \mathcal U) \to \mathbb R$ through $r(x^{0,N}_t, \mathbb G^N_t) \equiv \sum_{x \in \mathcal X} \tilde r(x, x^{0,N}_t, \mathbb G^N_t) \sum_{u \in \mathcal U} \mathbb G^N_t(x,u)$. 

As the optimality concept in this work, we therefore define approximate Pareto optimality.

\begin{definition}[Pareto optimality]
For $\epsilon > 0$, $\boldsymbol \pi^\epsilon \in \Pi_N^N$ is $\epsilon$-Pareto optimal if and only if
\begin{align}
    J^N(\boldsymbol \pi^\epsilon) \geq \sup_{\boldsymbol \pi} J^N(\boldsymbol \pi) - \epsilon \, .
\end{align}
\end{definition}

A visualization of this model can be found in Figure~\ref{fig:overview}.

\subsection{Mean Field Model}
As $N \to \infty$, we formally obtain the following mean field MDP, which will be rigorously justified in the sequel. At each time $t = 0, 1, \ldots$, the environment state is a random variable denoted by $x^{0}_t \in \mathcal X^0$. We consider Markovian upper-level policies $\pi = \{ \pi_t \}_{t \geq 0} \in \Pi$ from the space of such policies $\Pi$ where $\pi_t \colon \mathcal X^0 \to \mathcal H$. We equip both $\mathcal H$ and $\Pi$ with the supremum metric. As mentioned, the population state distribution is fixed to $\mu_0 \in \mathcal P(\mathcal X)$ at all times. The random state-action distribution is therefore given by
\begin{align}
    \mathbb G_t \coloneqq \mathbb G(\mu_0, \pi_t(x^{0}_t))
\end{align}
where $\mathbb G \colon \mathcal P(\mathcal X) \times \mathcal H \to \mathcal P(\mathcal X \times \mathcal U)$ is defined by
\begin{align}
    \mathbb G(\mu, h)(x,u) \coloneqq h(x,u) \mu(x)
\end{align}
for any $x \in \mathcal X, u \in \mathcal U$. The random environment state variables therefore follow $x_0^0 \sim \mu^0_0$ and subsequently
\begin{align}
    x^0_{t+1} &\sim P^0(x^0_t, \mathbb G_t) \, .
\end{align}
Analogously, the objective becomes
\begin{align}
    J(\pi) = \mathbb E \left[ \sum_{t=0}^{\infty} \gamma^t r(x^0_t, \mathbb G_t) \right] \, .
\end{align}

We require the following simple continuity assumption to obtain meaningful results in the limit as $N \to \infty$.

\begin{assumption}[Continuity of $r$ and $P^0$] \label{assumption}
The functions $r$ and $P^0$ are continuous, i.e. for all $x^0 \in \mathcal X^0$ and $\mathbb G_n \to \mathbb G \in \mathcal P(\mathcal X \times \mathcal U)$ we have
\begin{align}
    r(x^0, \mathbb G_n) \to r(x^0, \mathbb G) , \quad
    P^0(x^0, \mathbb G_n) \to P^0(x^0, \mathbb G) \, .
\end{align}
\end{assumption}

By compactness of $\mathcal P(\mathcal X \times \mathcal U)$, we have boundedness.

\begin{proposition}
Under Assumption~\ref{assumption}, $r$ is bounded by some $R$, i.e. for any $x^0 \in \mathcal X^0$, $\mathbb G \in \mathcal P(\mathcal X \times \mathcal U)$ we have
\begin{align}
    |r(x^0, \mathbb G)| \leq R \, .
\end{align}
\end{proposition}

Our first goal will be to show that as $N \to \infty$, the optimal solution to the MFC is approximately Pareto optimal in the finite $N$ case. This will motivate solving the MFC problem.

\section{APPROXIMATE OPTIMALITY}
We first show the following lemma on uniform convergence in probability of empirical state-action distributions to their state-action-wise average for fixed one-step policy profiles.

\begin{lemma} \label{lem:Ginprob}
Let $x^{0} \in \mathcal X^0$ and $\boldsymbol h \in \mathcal H^N$ be an arbitrary one-step policy profile. Let $\mathbb G^N \sim \mathcal G(\mu_0,\boldsymbol h)$. Then
\begin{enumerate}
    \item[(i)] $\mathbb{E}\left[\lVert \mathbb G^N - \mathbb{E}[\mathbb G^N]\rVert_1^2\right]\le \frac{|\mathcal X|^2 |\mathcal U|^2}{4N}$
    \item[(ii)] $\mathbb{P}\left(\lVert \mathbb G^N - \mathbb{E}[\mathbb G^N]\rVert_1 \geq \epsilon \right) \le \frac{|\mathcal X|^2 |\mathcal U|^2}{4\epsilon^2 N}$
\end{enumerate}
\end{lemma}
\begin{proof}
By Chebyshev's inequality, (i) implies (ii). It remains to prove (i). Let $x^{i,N}\sim \mu_0$ be i.i.d. and $u^{i,N}\sim \pi^i(x^{0}, x^{i,N})$, s.t. $(x^{i,N}, u^{i,N})_{i = 1, \dots, N}$ are independent. Then by the sub-additivity of $\mathbb E[(\cdot)^2]^\frac 1 2$, we have
\begin{align*}
    &\mathbb{E}\left[\lVert \mathbb G^N - \mathbb{E}[\mathbb G^N]\rVert_1^2\right]^\frac 1 2 \\
    &\quad = \mathbb{E}\left[\left(\sum_{x \in \mathcal X, u \in \mathcal U} \left\vert\frac 1 N \sum_{i = 1}^N\mathbf{1}_{x,u}(x^{i,N}, u^{i,N}) \right. \right. \right. \\
    &\qquad \qquad \left. \left. \left. -  \mathbb{E}\left[ \frac 1 N \sum_{i = 1}^N \mathbf{1}_{x,u}(x^{i,N}, u^{i,N})\right]\right\vert\right)^2\right]^\frac 1 2 \\
    &\quad \leq \sum_{x \in \mathcal X, u \in \mathcal U} \left(\mathbb{V}\left[\frac 1 N \sum_{i = 1}^N\mathbf{1}_{x,u}(x^{i,N}, u^{i,N})\right]\right)^{\frac 1 2} \\
    &\quad = \sum_{x \in \mathcal X, u \in \mathcal U} \left( \frac{1}{N^2} \sum_{i = 1}^N \mathbb{V}\left[\mathbf{1}_{x,u}(x^{i,N}, u^{i,N})\right]\right)^{\frac 1 2} \\
    &\quad \leq \sum_{x \in \mathcal X, u \in \mathcal U} \left( \frac{1}{N^2} \sum_{i = 1}^N \frac 1 4\right)^{\frac 1 2} = \frac{|\mathcal X| |\mathcal U|}{2 \sqrt N}
\end{align*}
using the trivial variance bound $\frac 1 4$ for indicator functions. 
\end{proof}

To achieve approximate optimality of mean field solutions in the $N$-agent case, we first define how to obtain an $N$-agent policy $\boldsymbol \pi^N \in \Pi_N^N$ from a mean field policy $\hat \pi \in \Pi$ by
\begin{align*}
    \boldsymbol \pi^N(\hat \pi) = (\pi^1, \ldots, \pi^N) \text{ with $\pi^i_t(x^0, x) = \hat \pi_t(x^0)(x)$}
\end{align*}
for all $i=1,\ldots,N$, i.e. all agents with state $x \in \mathcal X$ will follow the action distribution $\hat \pi_t(x^0_t)(x)$ at times $t \geq 0$.

\begin{theorem} \label{thm:uniformV}
Under Assumption~\ref{assumption}, we have uniform convergence of the $N$-agent objective to the mean field objective as $N \to \infty$, i.e.
\begin{align}
    \lim_{N \to \infty} \sup_{\pi \in \Pi} \left| J^N(\boldsymbol \pi^N(\pi)) - J(\pi) \right| = 0 \, .
\end{align}
\end{theorem}
\begin{proof}
We have by definition
\begin{align}
    &\sup_{\pi \in \Pi} \left| J^N(\boldsymbol \pi^N(\pi)) - J(\pi) \right| \\ 
    &\quad = \sup_{\pi \in \Pi} \left| \sum^{\infty}_{t=0} \gamma^t \mathbb E \left[ r(x^{0,N}_t, \mathbb G_t^N) - r(x^0_t, \mathbb G_t) \right] \right| \\
    &\quad \leq \sum^{\infty}_{t=0} \gamma^t \sup_{\pi \in \Pi} \left| \mathbb E \left[ r(x^{0,N}_t, \mathbb G_t^N) - r(x^0_t, \mathbb G_t) \right] \right| \, . \label{eq:domconv}
\end{align}

To obtain the desired result, we first show for any $t \geq 0$ that $\sup_{\pi \in \Pi} \lVert \mathcal L(x^{0,N}_t) - \mathcal L(x^{0}_t) \rVert_1 \to 0$ implies $\mathcal L(x^{0,N}_t, \mathbb G_t^N) \to \mathcal L(x^{0}_t, \mathbb G_t)$ weakly uniformly over all $\pi \in \Pi$. Note that $\sup_{\pi \in \Pi} \lVert \mathcal L(x^{0,N}_t) - \mathcal L(x^{0}_t) \rVert_1 \to 0$ by definition implies
\begin{align*}
    \sup_{\pi \in \Pi} \left| \mathcal L(x^{0,N}_t)(x^0) - \mathcal L(x^{0}_t)(x^0) \right| \to 0
\end{align*}
for any $x^0 \in \mathcal X^0$. For the joint law, consider any $f \colon \mathcal X^0 \times \mathcal P(\mathcal X \times \mathcal U) \to \mathbb R$,  continuous and bounded by $|f| \leq F$. Then
\begin{align*}
    &\sup_{\pi \in \Pi} \left| \mathcal L(x^{0,N}_t, \mathbb G_t^N)(f) - \mathcal L(x^{0}_t, \mathbb G_t)(f) \right| \\
    &\quad = \sup_{\pi \in \Pi} \left| \mathbb E \left[ f(x^{0,N}_t, \mathbb G_t^N) \right] - \mathbb E \left[ f(x^{0}_t, \mathbb G_t) \right] \right| \\
    &\quad \leq \sup_{\pi \in \Pi} \sum_{x^0 \in \mathcal X^0} \left| \mathbb E \left[ f(x^{0,N}_t, \mathbb G_t^N) \mid x^{0,N}_t = x^0 \right] \mathcal L(x^{0,N}_t)(x) \right. \\
    &\qquad \qquad \qquad - \left. f(x^0, \mathbb G(\mu_0, \pi_t(x^0)) \, \mathcal L(x^{0}_t)(x^0) \right| \\
    &\quad \leq \sum_{x^0 \in \mathcal X^0} \sup_{\pi \in \Pi} \left| \mathbb E \left[ f(x^{0,N}_t, \mathbb G_t^N) \mid x^{0,N}_t = x^0 \right] \right| \\ 
    &\qquad \qquad \qquad \cdot \sup_{\pi \in \Pi} \left| \mathcal L(x^{0,N}_t)(x^0) - \mathcal L(x^{0}_t)(x^0) \right| \\
    &\qquad + \sum_{x^0 \in \mathcal X^0} \sup_{\pi \in \Pi} \bigg| f(x^0, \mathbb G(\mu_0, \pi_t(x^0)) \\
    &\qquad \qquad \qquad - \mathbb E \left[ f(x^{0,N}_t, \mathbb G_t^N) \mid x^{0,N}_t = x^0 \right] \bigg| \\
    &\qquad \qquad \qquad \cdot \sup_{\pi \in \Pi} \mathcal L(x^{0}_t)(x^0),
\end{align*}
where the first sum goes to zero by assumption and boundedness of $f$. For the second term, consider arbitrary fixed $x^0 \in \mathcal X^0$. Write $\mathbb G_\pi$ short for $\mathbb G(\mu_0, \pi_t(x^0))$ and introduce
$\mathbb G^N_\pi \sim \mathcal G(\mu_0, (\pi_t(x^0), \dots, \pi_t(x^0)))$
for all $N, \pi \in \Pi$. So in contrast to $\mathbb G_t^N$ that depends on a random $x_t^{0,N}$, the random probability mass function $\mathbb G^N_\pi$ has $x_t^{0,N}= x^0$ fixed. Then
\begin{align*}
     &f(x^0, \mathbb G(\mu_0, \pi_t(x^0))) - \mathbb E \left[ f(x^{0,N}_t, \mathbb G_t^N) \mid x^{0,N}_t = x^0 \right] \\
     &= f(x^0, \mathbb G_\pi) - \mathbb E \left[ f(x^0, \mathbb G^N_\pi) \right]
\end{align*}
We observe that for any $(x,u) \in \mathcal X \times \mathcal U$
\begin{equation}
     \mathbb{E}[\mathbb G^N_\pi(x,u)] = \mathbb G_\pi(x,u).
     \label{eq:expectation}
\end{equation}
For this purpose, let $x^{i,N}\sim \mu_0$ be i.i.d. and $u^{i,N}\sim \pi_t(x^{0}, x^{i,N})$, s.t. $(x^{i,N}, u^{i,N})_{i = 1, \dots, N}$ are independent. Then for any $(x,u)\in \mathcal X \times \mathcal U$ we have
\begin{align*}
    \mathbb{E}[\mathbb G^N_\pi(x,u)] &= \frac 1 N \sum_{i = 1}^N \mathbb E \left[ \mathbf{1}_{x,u}(x^{i,N}, u^{i,N})\right]\\
    &= \frac 1 N \sum_{i = 1}^N \mu_0(x)\pi_t(x^0,x,u)\\
    &= \mathbb G_\pi(x,u) \, .
\end{align*}
Let $\epsilon > 0$ arbitrary. By compactness of $\mathcal P(\mathcal X \times \mathcal U)$, the function $f(x^0,\cdot)\colon P(\mathcal X \times \mathcal U) \to \mathbb R$ is uniformly continuous. Consequently, there exists $\delta > 0$ such that for all $\pi\in \Pi$
\begin{align*}
    &\lVert \mathbb G_\pi - \mathbb G^N_\pi \rVert_1 < \delta\\ &\qquad \qquad\implies \left| f(x^0, \mathbb G_\pi) - f(x^0, \mathbb G^N_\pi) \right| < \frac{\epsilon}{2} \, .
\end{align*}
By Lemma~\ref{lem:Ginprob} (ii) and \eqref{eq:expectation} there exists $N' \in \mathbb N$ such that for $N > N'$ and for all $\pi \in \Pi$ we have
\begin{align*}
    \mathbb P \left( \lVert \mathbb G_\pi - \mathbb G^N_\pi \rVert_1 \geq \delta \right) \leq \frac{\epsilon}{4F} \, .
\end{align*}
As a result, we have
\begin{align*}
    & \mathbb E \left[ \left| f(x^0, \mathbb G_\pi) - f(x^0, \mathbb G^N_\pi) \right|\right] \\
    &\quad \leq \mathbb P \left( \left| f(x^0, \mathbb G_\pi) - f(x^0, \mathbb G^N_\pi \right) \right| \geq \frac{\epsilon}{2} ) \cdot 2F + 1 \cdot \frac{\epsilon}{2} \\
    &\quad \leq \mathbb P \left( \lVert \mathbb G_\pi - \mathbb G^N_\pi \rVert_1 \geq \delta \right) \cdot 2F + \frac{\epsilon}{2} \\
    &\quad \leq \frac{\epsilon}{4F} \cdot 2F + \frac{\epsilon}{2} = \epsilon \, .
\end{align*}
Since $\epsilon$ was arbitrary, and no choices depended on $\pi \in \Pi$, we have the desired convergence of the second term
\begin{align*}
    &\lim_{N \to \infty}\sup_{\pi \in \Pi} \bigg| f(x, \mathbb G(\mu_0, \pi_t(x)) \\
    &\qquad \qquad \qquad - \mathbb E \left[ f(x^{0,N}_t, \mathbb G_t^N) \mid x^{0,N}_t = x \right] \bigg| = 0 \, .
\end{align*}

We can now show $\mathcal L(x^{0,N}_t, \mathbb G_t^N) \to \mathcal L(x^{0}_t, \mathbb G_t)$ weakly uniformly over all $\pi \in \Pi$ by induction over all $t$, which by Assumption~\ref{assumption} will imply 
\begin{align}
    \sup_{\pi \in \Pi} \left| \mathbb E \left[ r(x^{0,N}_t, \mathbb G_t^N) - r(x^0_t, \mathbb G_t) \right] \right| \to 0
\end{align}
for all $t \geq 0$ and hence the desired statement by the dominated convergence theorem applied to \eqref{eq:domconv}.

At $t=0$, we trivially have $\mathcal L(x^{0,N}_t) = \mu_0^0 = \mathcal L(x^{0}_t)$ and therefore $\mathcal L(x^{0,N}_0, \mathbb G_0^N) \to \mathcal L(x^{0}_0, \mathbb G_0)$ uniformly by the prequel. Assume that the induction assumption holds at time $t$, then at time $t+1$ we have
\begin{align*}
    &\lVert \mathcal L(x_{t+1}^{0,N}) - \mathcal L(x_{t+1}^{0}) \rVert_1 \\
    &\quad = \sum_{x^0 \in \mathcal X^0} \left| \mathcal L(x_{t+1}^{0,N})(x^0) - \mathcal L(x_{t+1}^{0})(x^0) \right| \\
    &\quad = \sum_{x^0 \in \mathcal X^0} \left| \mathbb E \left[ P^0(x^0 \mid x^{0,N}_t, \mathbb G_t^N) \right] - \mathbb E \left[ P^0(x^0 \mid x^{0}_t, \mathbb G_t) \right] \right| \\
    &\quad \to 0
\end{align*}
uniformly by Assumption~\ref{assumption} and induction assumption.
\end{proof}

To extend to optimality over arbitrary asymmetric policy tuples, we show that the performance of policy tuples is close to the averaged policy as $N \to \infty$.

\begin{theorem} \label{thm:uniformAvg}
Under Assumption~\ref{assumption}, as $N \to \infty$ we have similar performance of any policy tuple $\boldsymbol \pi = (\pi^1, \ldots, \pi^N) \in \Pi_N^N$ and its average policy $\hat \pi(\boldsymbol \pi) \in \Pi$ defined by $\hat \pi_t(x^0)(a \mid x) = \frac{1}{N} \sum_{i=1}^N \pi^i_t(a \mid x^0, x)$ in the $N$-agent case, i.e. with shorthand $\hat \pi = \hat \pi(\boldsymbol \pi)$ we have
\begin{align}
    \lim_{N \to \infty} \sup_{\boldsymbol \pi \in \Pi_N^N} \left| J^N(\pi^1, \ldots, \pi^N) - J^N(\boldsymbol \pi^N(\hat \pi)) \right| = 0 \, .
\end{align}
\end{theorem}
\begin{proof}
Let $\boldsymbol \pi \in \Pi_N^N$ arbitrary. Again, we have by definition
\begin{align}
    &\sup_{\boldsymbol \pi \in \Pi_N^N} \left| J^N(\pi^1, \ldots, \pi^N) - J^N(\boldsymbol \pi^N(\hat \pi)) \right| \\ 
    &\quad \leq \sum^{\infty}_{t=0} \gamma^t \sup_{\boldsymbol \pi \in \Pi_N^N} \left| \mathbb E \left[ r(x^{0,N}_t, \mathbb G_t^N) - r(\hat x^{0,N}_t, \hat {\mathbb G}_t^N) \right] \right| \label{eq:thm3terms}
\end{align}
by introducing random variables $\hat x^{0,N}_t$, $\hat {\mathbb G}_t^N$, $\hat x^{i,N}_t$, $\hat u^{i,N}_t$, $i = 1, \ldots, N$ induced by instead applying the averaged policy tuple $\boldsymbol \pi^N(\hat \pi)$ in \eqref{eq:finiteu}. By dominated convergence, it is sufficient to show term-wise convergence to zero in \eqref{eq:thm3terms}.

Fix $t \geq 0$. As in the proof of Theorem~\ref{thm:uniformV}, we show that $\sup_{\pi \in \Pi} \lVert \mathcal L(x^{0,N}_t) - \mathcal L(\hat x^{0,N}_t) \rVert_1 \to 0$ implies $\sup_{\boldsymbol \pi \in \Pi_N^N} \left| \mathcal L(x^{0,N}_t, \mathbb G_t^N)(f) - \mathcal L(\hat x^{0,N}_t, \hat{\mathbb G}_t^N)(f) \right| \to 0$ for any $f \colon \mathcal X^0 \times \mathcal P(\mathcal X \times \mathcal U) \to \mathbb R$ continuous and bounded, since
\begin{align*}
    &\sup_{\boldsymbol \pi \in \Pi_N^N} \left| \mathcal L(x^{0,N}_t, \mathbb G_t^N)(f) - \mathcal L(\hat x^{0,N}_t, \hat{\mathbb G}_t^N)(f) \right| \\
    &\quad = \sup_{\boldsymbol \pi \in \Pi_N^N} \left| \mathbb E \left[ f(x^{0,N}_t, \mathbb G_t^N) \right] - \mathbb E \left[ f(\hat x^{0,N}_t, \hat{\mathbb G}_t^N) \right] \right| \\
    &\quad \leq \sum_{x^0 \in \mathcal X^0} \sup_{\boldsymbol \pi \in \Pi_N^N} \left| \mathbb E \left[ f(x^{0,N}_t, \mathbb G_t^N) \mid x^{0,N}_t = x^0 \right] \right| \\ 
    &\qquad \qquad \qquad \cdot \sup_{\boldsymbol \pi \in \Pi_N^N} \left| \mathcal L(x^{0,N}_t)(x^0) - \mathcal L(\hat x^{0,N}_t)(x^0) \right| \\
    &\qquad + \sum_{x^0 \in \mathcal X^0} \sup_{\boldsymbol \pi \in \Pi_N^N} \bigg| \mathbb E \left[ f(x^0, \hat{\mathbb G}_t^N) \mid \hat x^{0,N}_t = x^0 \right] \\
    &\qquad \qquad \qquad - \mathbb E \left[ f(x^0, \mathbb G_t^N) \mid x^{0,N}_t = x^0 \right] \bigg| \\
    &\qquad \qquad \qquad \cdot \sup_{\boldsymbol \pi \in \Pi_N^N} \mathcal L(\hat x^{0,N}_t)(x^0)
\end{align*}
where the first sum goes to zero by assumption and boundedness of $f$. For the second term, consider arbitrary fixed $x^0 \in \mathcal X^0$, $\boldsymbol \pi \in \Pi_N^N$.
Then introduce random variables $\mathbb G^N_{\boldsymbol\pi}\sim \mathcal G^N(\mu_0,(\pi_t^1(x^0), \dots, \pi_t^N(x^0)))$ and $\mathbb G^N_{\hat \pi}\sim \mathcal G^N(\mu_0,(\hat\pi_t(x^0),\dots, \hat\pi_t(x^0)) )$ for every $N \in \mathbb N$ and $\boldsymbol \pi \in \Pi^N_N$.
Then we have
\begin{align*}
     &\mathbb E \left[ f(x^0, \hat{\mathbb G}_t^N) \mid \hat x^{0,N}_t = x^0 \right]- \mathbb E \left[ f(x^0, \mathbb G_t^N) \mid x^{0,N}_t = x^0 \right] \\
     &\quad = \mathbb E \left[f(x^0, \mathbb G^N_{\hat \pi}\right] - \mathbb E \left[ f(x^0, \mathbb G^N_{\boldsymbol\pi}) \right] \, .
\end{align*}
We observe that for any $(x,u) \in \mathcal X \times \mathcal U$:
\begin{equation}
     \mathbb{E}[\mathbb G^N_{\hat \pi}(x,u)] =  \mathbb{E}[\mathbb G^N_{\boldsymbol\pi}(x,u)] \, .
     \label{eq:expectation2}
\end{equation}
For this purpose, let $x^{i,N}\sim \mu_0$ and $u^{i,N}\sim \pi^i_t(x^0, x^{i,N})$ as well as $\hat x^{i,N}\sim \mu_0$ and $\hat u^{i,N}\sim \hat\pi_t(x^0, x^{i,N})$, s.t. $(x^{i,N}, u^{i,N})_{i = 1, \dots, N}$ and $(\hat x^{i,N}, \hat u^{i,N})_{i = 1, \dots, N}$ are independent, respectively. Then for any $(x,u) \in \mathcal X \times \mathcal U$:
\begin{align*}
    &\mathbb{E}[\mathbb G^N_{\hat \pi} (x,u)] = \frac{1}{N}\sum_{i = 1}^N \mathbb{P}(\hat x^{i,N}=x, \hat u^{i,N} = u)\\
    &\quad = \frac{1}{N}\sum_{i = 1}^N \mu_0(x) \hat\pi(x,u)= \frac{1}{N}\sum_{i = 1}^N \mu_0(x) \frac{1}{N}\sum_{j=1}^N \pi^j(x,u)\\
    &\quad = \frac{1}{N}\sum_{j = 1}^N \mu_0(x) \pi^j(x,u)= \frac{1}{N}\sum_{j = 1}^N \mathbb{P}(x^{i,N}=x, u^{i,N} = u)\\
    &\quad = \mathbb{E}[\mathbb G^N_{\boldsymbol\pi}(x,u)] \, .
\end{align*}
Then by \eqref{eq:expectation2}, sub-additivity of $\mathbb E[(\cdot)^2]^\frac 1 2$ and Lemma~\ref{lem:Ginprob} (i),
\begin{align*}
     &\mathbb{E}[\lVert \mathbb G^N_{\hat \pi} - \mathbb G^N_{\boldsymbol\pi} \rVert_1^2]^{\frac 12}\\
     &\le \mathbb{E}\left[\left(\lVert \mathbb G^N_{\hat \pi} - \mathbb{E} [\mathbb G^N_{\hat \pi}]\rVert_1 +  \lVert \mathbb G^N_{\boldsymbol\pi} - \mathbb{E}[\mathbb G^N_{\boldsymbol\pi}]\rVert_1\right)^2\right]^{\frac 12}\\
     &\le \mathbb{E}[\lVert \mathbb G^N_{\hat \pi} - \mathbb{E} [\mathbb G^N_{\hat \pi}]\rVert_1^2]^{\frac 1 2} + \mathbb{E}[\lVert \mathbb G^N_{\boldsymbol\pi} - \mathbb{E} [\mathbb G^N_{\boldsymbol\pi}]\rVert_1^2]^{\frac 1 2}\\
     &\le 2 \frac{\vert \mathcal X \vert \cdot \vert \mathcal U \vert}{\sqrt{4N}} = \frac{\vert \mathcal X \vert \cdot \vert \mathcal U \vert}{\sqrt{N}} \, .
\end{align*}
Chebyshev's inequality implies
\begin{align}
     \mathbb{P}\left(\lVert \mathbb G^N_{\hat \pi} - \mathbb G^N_{\boldsymbol\pi} \rVert_1 \geq \epsilon \right) \le \frac{|\mathcal X|^2 |\mathcal U|^2}{\epsilon^2 N}
\end{align}
independent of $\boldsymbol \pi \in \Pi_N^N$.

Then analogously to the proof of Theorem~\ref{thm:uniformV},
\begin{align*}
    &\sup_{\boldsymbol \pi \in \Pi_N^N}\bigg| \mathbb E \left[ f(x^0, \hat{\mathbb G}_t^N) \mid \hat x^{0,N}_t = x^0 \right] \\
    &\qquad \qquad \qquad - \mathbb E \left[ f(x^0, \mathbb G_t^N) \mid x^{0,N}_t = x^0 \right] \bigg| \to 0
\end{align*}
can be concluded, showing the desired implication.

We now show by induction over all $t$ that for any $t \geq 0$, and any $f \colon \mathcal X^0 \times \mathcal P(\mathcal X \times \mathcal U) \to \mathbb R$ continuous and bounded, $\sup_{\boldsymbol \pi \in \Pi_N^N} \left| \mathcal L(x^{0,N}_t, \mathbb G_t^N)(f) - \mathcal L(\hat x^{0,N}_t, \hat{\mathbb G}_t^N)(f) \right| \to 0$ which by Assumption~\ref{assumption} will again imply that \eqref{eq:thm3terms} goes to zero. 

At $t=0$, we have $\mathcal L(x^{0,N}_0) = \mu_0^0 = \mathcal L(\hat x^{0,N}_0)$, implying $\sup_{\boldsymbol \pi \in \Pi_N^N} \left| \mathcal L(x^{0,N}_0, \mathbb G_0^N)(f) - \mathcal L(\hat x^{0,N}_0, \hat{\mathbb G}_0^N)(f) \right| \to 0$ for any $f \colon \mathcal X^0 \times \mathcal P(\mathcal X \times \mathcal U) \to \mathbb R$ by the prequel. Assuming the induction assumption holds at time $t$, then at time $t+1$
\begin{align*}
    &\lVert \mathcal L(x_{t+1}^{0,N}) - \mathcal L(\hat x_{t+1}^{0,N}) \rVert_1 \\
    &\quad = \sum_{x^0 \in \mathcal X^0} \left| \mathbb E \left[ P^0(x^0 \mid x^{0,N}_t, \mathbb G_t^N) - P^0(x^0 \mid \hat x_{t}^{0,N}, \hat{\mathbb G}_t^N) \right] \right| \\ 
    &\quad \to 0
\end{align*}
uniformly by induction assumption and continuity and boundedness of $P^0$, which implies the desired statement.
\end{proof}

\begin{corollary}
Under Assumption~\ref{assumption}, for any $\epsilon > 0$ there exists $N(\epsilon)$ such that for all $N > N(\epsilon)$ a policy $\pi^*$ optimal in the MFC MDP -- that is, $J(\pi^*) = \sup_{\pi \in \Pi} J(\pi)$ -- is $\epsilon$-Pareto optimal in the $N$-agent case, i.e. 
\begin{align}
    J^N(\boldsymbol \pi^N(\pi^*)) \geq \sup_{\boldsymbol \pi} J^N(\boldsymbol \pi) - \epsilon \, .
\end{align} 
\end{corollary}
\begin{proof}
By Theorem~\ref{thm:uniformV} and Theorem~\ref{thm:uniformAvg}, there exists $N' \in \mathbb N$ such that for average policy $\hat \pi$ of $\boldsymbol \pi$ and all $N > N'$ we have
\begin{align*}
    &\sup_{\boldsymbol \pi} \left( J^N(\boldsymbol \pi) - J^N(\boldsymbol \pi^N(\pi^*)) \right) \\
    &\quad \leq \sup_{\boldsymbol \pi} \left( J^N(\boldsymbol \pi) - J^N(\boldsymbol \pi^N(\hat \pi)) \right) \\
    &\qquad + \sup_{\boldsymbol \pi} \left( J^N(\boldsymbol\pi^N(\hat \pi)) - J(\hat \pi) \right) \\
    &\qquad + \sup_{\boldsymbol \pi} \left( J(\hat \pi) - J(\pi^*) \right) \\
    &\qquad + \sup_{\boldsymbol \pi} \left( J(\pi^*) - J^N(\boldsymbol \pi^N(\pi^*)) \right) \\
    &\quad < \frac{\epsilon}{3} + \frac{\epsilon}{3} + 0 + \frac{\epsilon}{3} = \epsilon \, .
\end{align*}
Reordering terms gives the desired inequality.
\end{proof}

\section{DYNAMIC PROGRAMMING PRINCIPLE}
The following dynamic programming principle for the MFC MDP is a standard result, for which the MDP state will be only the environment state, see e.g. \cite{motte2019mean, gu2020q}.

Define action-value function $Q \colon \mathcal X^0 \times \mathcal H \to \mathbb R$,
\begin{multline}
    Q(x^0, h) \coloneqq \sup_{\pi \in \Pi} \mathbb E \left[ \sum_{t=0}^\infty \gamma^t r(x^0_t, \mathbb G(\mu_0, \pi_t(x^0_t))) \right. \\ \left. \mid x^0_0 = x^0, \pi_0(x^0) = h \right] \, .
\end{multline} 
Note that by boundedness of $r$, we trivially have
\begin{align*}
    \left| Q \right| \leq \frac{R}{1-\gamma} \, .
\end{align*}

As we have an MDP with finite state space $\mathcal X^0$, the following Bellman equation will hold, see \cite{puterman2014markov}.

\begin{theorem}
The Bellman equation
\begin{multline}
    Q(x^0, h) = r(x^0, \mathbb G(\mu_0, h)) \\+ \gamma \mathbb E_{\tilde x^0 \sim P^0(x^0, \mathbb G(\mu_0, h))} \left[ \sup_{\tilde h \in \mathcal H} Q(\tilde x^0, \tilde h) \right]
\end{multline}
holds for all $x^0 \in \mathcal X^0, h \in \mathcal H$.
\end{theorem}

In the following, we obtain existence of an optimal stationary policy by compactness of $\mathcal H$ and continuity of $Q$, which shall be inherited from the continuity of $r$ and $P^0$.

\begin{lemma} \label{lem:unique-q}
The unique function that satisfies the Bellman equation is given by $Q$. Further, if there exists $h_{x^0} \in \argmax_{h \in \mathcal H} Q(x^0, h)$ for any $x^0 \in \mathcal X^0$, then the policy $\pi^*$ with $\pi^*_t(x^0) = h_{x^0}$ is an optimal stationary policy.
\end{lemma}
\begin{proof}
For uniqueness, define the space of $\frac{R}{1-\gamma}$-bounded functions $\mathcal Q \coloneqq \{ f \colon \mathcal X^0 \times \mathcal H \to [-\frac{R}{1-\gamma}, \frac{R}{1-\gamma}] \}$ and the Bellman operator $B \colon \mathcal Q \to \mathcal Q$ defined by
\begin{multline}
    (BQ)(x^0, h) = r(x^0, \mathbb G(\mu_0, h)) \\+ \gamma \mathbb E_{\tilde x^0 \sim P^0(x^0, \mathbb G(\mu_0, h))} \left[ \sup_{\tilde h \in \mathcal H} Q(\tilde x^0, \tilde h) \right] \, .
\end{multline}
We show that $\mathcal Q$ is a complete metric space under the supremum norm. Let $(Q_n)_{n \in \mathbb N}$ be a Cauchy sequence of functions $Q_n \in \mathcal Q$. Then by definition, for any $\epsilon > 0$ there exists $n' \in \mathbb N$ such that for all $n, m > n'$ we have
\begin{align*}
    \lVert Q_n - Q_m \rVert_\infty < \epsilon \\
    \implies \forall x^0 \in \mathcal X^0, h \in \mathcal H \colon \left| Q_n(x^0, h) - Q_m(x^0, h) \right| < \epsilon
\end{align*}
such that for all $x^0 \in \mathcal X^0, h \in \mathcal H$ there exists a value $c_{x^0, h} \in [-\frac{R}{1-\gamma}, \frac{R}{1-\gamma}]$ for which $Q_n(x^0, h) \to c_{x^0, h}$. Define the function $Q' \in \mathcal Q$ by $Q'(x^0, h) = c_{x^0, h}$, then we have
\begin{align*}
    &\left| Q_n(x^0, h) - Q'(x^0, h) \right| \\
    &\quad = \lim_{m \to \infty} \left| Q_n(x^0, h) - Q_m(x^0, h) \right| < \epsilon 
\end{align*}
for all $x^0 \in \mathcal X^0, h \in \mathcal H, n > n'$, and hence $Q_n \to Q' \in \mathcal Q$ as $n \to \infty$. This implies completeness of $(\mathcal Q, \lVert \cdot \rVert_\infty)$.

We now show that $B$ is a contraction under the supremum norm, i.e.
\begin{align*}
    \lVert BQ_1 - BQ_2 \rVert_\infty \leq C \lVert Q_1 - Q_2 \rVert_\infty
\end{align*}
for some $C < 1$. Define the shorthand $\tilde x^0 \sim P^0(x^0, \mathbb G(\mu_0, h))$. We have
\begin{align*}
    &\lVert BQ_1 - BQ_2 \rVert_\infty \\
    &\quad = \sup_{x^0 \in \mathcal X^0, h \in \mathcal H} \left| BQ_1(x^0, h) - BQ_2(x^0, h) \right| \\
    &\quad \leq \sup_{x^0 \in \mathcal X^0, h \in \mathcal H} \gamma \mathbb E_{\tilde x^0} \left[ \left| \sup_{\tilde h \in \mathcal H} Q_1(\tilde x^0, \tilde h) - \sup_{\tilde h \in \mathcal H} Q_2(\tilde x^0, \tilde h) \right| \right] \\
    &\quad \leq \sup_{x^0 \in \mathcal X^0, h \in \mathcal H} \gamma \lVert Q_1 - Q_2 \rVert_\infty
\end{align*}
with $\gamma < 1$. Therefore, by Banach fixed point theorem, $B$ has the unique fixed point $Q$.

For optimality, define the policy action-value function $Q^\pi$ for $\pi \in \Pi$ as the fixed point of $B^\pi \colon \mathcal Q \to \mathcal Q$ defined by
\begin{align*}
    (B^\pi Q)(x^0, h) = r(x^0, \mathbb G(\mu_0, h)) + \gamma \mathbb E_{\tilde x^0} \left[ Q(\tilde x^0, \pi(\tilde x^0)) \right] \, .
\end{align*}
From this, we immediately have
\begin{align*}
    Q^{\pi^*}(x^0, h) 
    &= r(x^0, \mathbb G(\mu_0, h)) + \gamma \mathbb E_{\tilde x^0} \left[ Q(\tilde x^0, \pi^*(\tilde x^0)) \right] \\
    &= r(x^0, \mathbb G(\mu_0, h)) + \gamma \mathbb E_{\tilde x^0} \left[ \sup_{\tilde h \in \mathcal H} Q(\tilde x^0, \tilde h) \right] \\
    &= Q(x^0, h)
\end{align*}
which implies that $\pi^*$ is optimal, see also \cite{puterman2014markov}.
\end{proof}

\begin{lemma} \label{lem:qcont}
The action-value function $Q$ is continuous.
\end{lemma}
\begin{proof}
We will show as $x^0_n \to x^0 \in \mathcal X^0$ and $h_n \to h \in \mathcal H$,
\begin{align*}
    Q(x^0_n, h_n) \to Q(x^0, h) \, .
\end{align*}
By the Bellman equation, we immediately have
\begin{align*}
    &| Q(x^0_n, h_n) - Q(x^0, h) | \\
    &\quad \leq \left| r(x^0_n, \mathbb G(\mu_0, h_n)) - r(x^0, \mathbb G(\mu_0, h)) \right| \\
    &\qquad + \left| \gamma \sum_{\tilde x^0 \in \mathcal X^0} \left( P^0(\tilde x^0 \mid x^0_n, \mathbb G(\mu_0, h_n)) \right.\right. \\
    &\qquad \qquad \qquad \left.\left. - P^0(\tilde x^0 \mid x^0, \mathbb G(\mu_0, h)) \right) \sup_{\tilde h \in \mathcal H} Q(\tilde x^0, \tilde h)  \right| \\
    &\quad \leq \left| r(x^0_n, \mathbb G(\mu_0, h_n)) - r(x^0, \mathbb G(\mu_0, h)) \right| \\
    &\qquad +  \frac{\gamma R}{1-\gamma} \sum_{\tilde x^0 \in \mathcal X^0} \left| P^0(\tilde x^0 \mid x^0_n, \mathbb G(\mu_0, h_n)) \right. \\
    &\qquad \qquad \qquad \left. - P^0(\tilde x^0 \mid x^0, \mathbb G(\mu_0, h)) \right| \to 0
\end{align*}
since $r, P^0, \mathbb G$ are continuous and $Q$ is bounded.
\end{proof}

\begin{corollary}
There exists an optimal stationary policy $\pi^*\colon \mathcal X^0 \to \mathcal H$ such that $Q^{\pi^*} = Q$.
\end{corollary}
\begin{proof}
By Lemma~\ref{lem:qcont}, $Q$ is continuous. Furthermore, $\mathcal H$ is compact. By the extreme value theorem, there exists $h_{x^0} \in \argmax_{h \in \mathcal H} Q(x^0, h)$ for any $x^0 \in \mathcal X^0$. By Lemma~\ref{lem:unique-q}, there exists an optimal stationary policy $\pi^*$.
\end{proof}

Since $\mathcal H$ is continuous, general exact solutions are difficult. Instead, we apply reinforcement learning with stochastic policies to find an optimal stationary policy.

\section{EXPERIMENTS}
We compare the empirical performance of the mean field solution in the aforementioned scheduling problem. Since there exist few theoretical guarantees for tractable multi-agent reinforcement learning methods \cite{zhang2021multi}, we compare our approach (MF) to empirically effective independent learning (IL) \cite{tan1993multi}, i.e. applying single-agent RL to each separate agent (NA), as well as the well-known Join-Shortest-Queue (JSQ) algorithm \cite{jsq_mf}, where agents choose the shortest queue accessible and otherwise randomly. To make independent learning more tractable, we also share policy parameters between all agents using parameter sharing (PS) \cite{gupta2017cooperative} and train each policy via the PPO algorithm \cite{schulman2017proximal} using the RLlib 1.2.0 Pytorch implementation \cite{liang2018rllib} for $400,000$ time steps in the $N$-agent cases and $2$ million time steps in the MF case, which is sufficient for convergence of MF and $N$-agent policies up to $N=4$, after which $N$-agent training becomes unstable under the shared hyperparameters in Table~\ref{table:hyperparameters} and continues to fail even with more time steps.

For policies and critics, we use separate feedforward networks with two hidden layers of 256 nodes and $\tanh$ activations. In the mean field case the policy outputs parameters $\boldsymbol \mu, \boldsymbol \sigma$ of a diagonal Gaussian distribution over actions, which are sampled and clipped between $0$ and $1$. We normalize each of these output values such that they give the probability of assigning to an accessible queue given some agent state, i.e. a shared lower-level decision rule $h \in \mathcal H$ for all agents. A visualization of this process can be found in Figure~\ref{fig:overview-impl}. 

Note that we use stochastic policies as required by stochastic policy gradient methods, though we can easily obtain a deterministic policy if necessary by simply using the mean parameter of the Gaussian distribution. In the $N$-agent case, we output queue assignment probabilities for each of the agents via a standard softmax final layer. Invalid assignments to queues that are not accessible by an agent are treated as randomly sampling one from all accessible queues.

\begin{table}
\centering
\caption{Parameter and hyperparameter settings used in the experiments of this work.}
\label{table:hyperparameters}
\begin{tabular}{@{}ccc@{}}
\toprule
Symbol     & Function          & Value     \\ 
\midrule
$c_d$ &   Packet drop penalty &  $1$ \\
$M$ &   Number of queues &  $2$ \\
$B_i$ &   Queue buffer sizes &  $5$ \\
$\Delta T$ &    Time step size   & \SI{0.5}{\second}    \\
$\lambda$ &   Packet arrival rate &  $(3M-1)$ \, \si{\per\second} \\
$\beta$ &   Queue servicing rate &  \SI{3}{\per\second} \\
$\gamma$ &   Discount factor &  $0.99$ \\
\midrule
        &  PPO                  & \\ 
\midrule
$l_{r}$ &   Learning rate & \num{5e-5}\\
$\lambda_\mathrm{PPO}$ &   GAE coefficient & $0.2$ \\
$\beta_\mathrm{PPO}$ &   Initial KL coefficient & $0.2$ \\
$d_\mathrm{targ}$ &   KL target & $0.01$ \\
$\epsilon$ &  Clip parameter & $0.3$ \\
$B$ &  Training batch size & $4000$ \\
$B_{m}$ & SGD mini-batch size & $128$ \\
$k$ &   SGD iterations per batch & $30$ \\
\bottomrule
\end{tabular}
\end{table}

\begin{figure}
    \centering
    \includegraphics[width=0.75\linewidth]{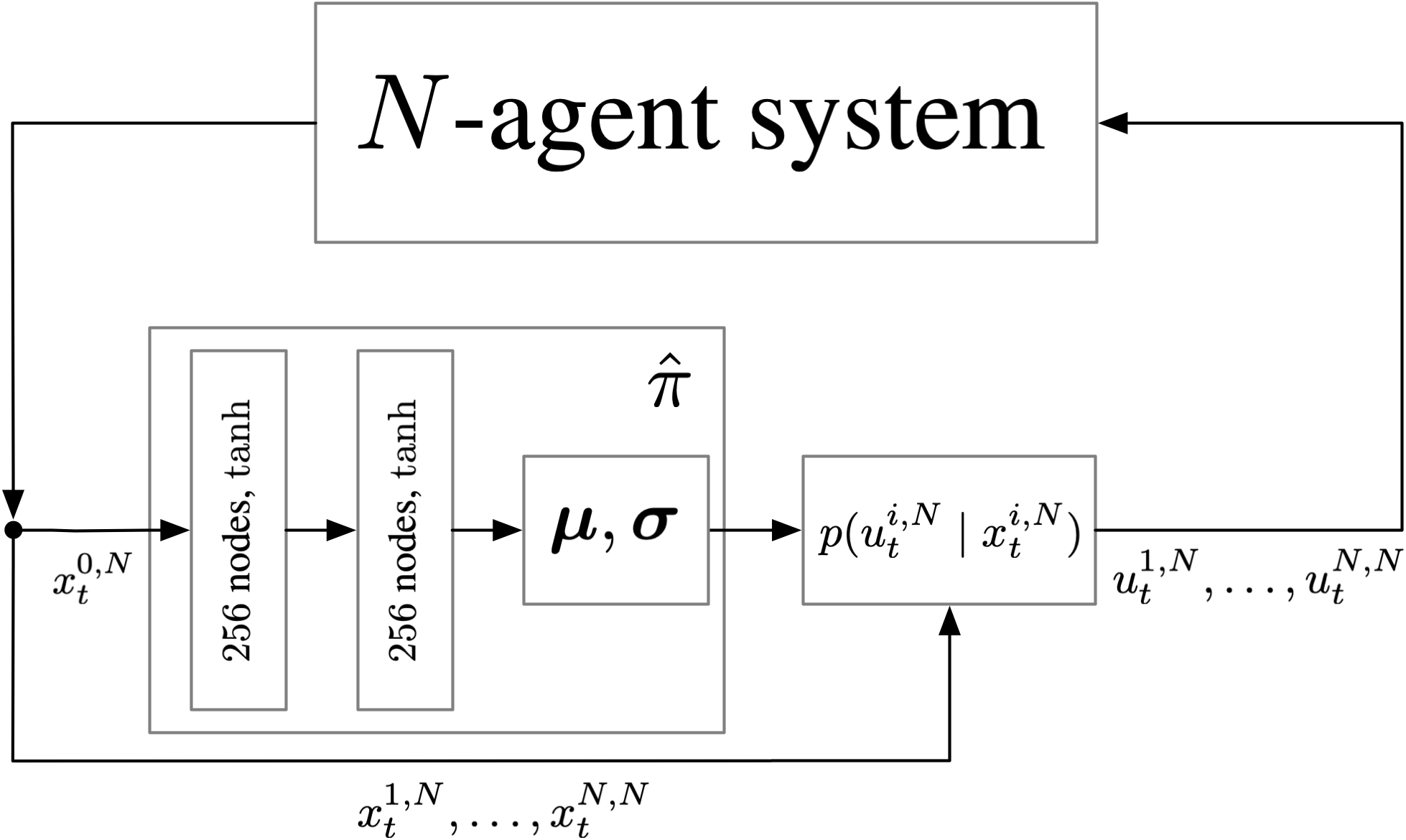}
    \caption{Overview of mean field control application in $N$-agent systems: Conditional on the environment state $x_t^{0,N}$, the upper-level mean field policy $\hat \pi$ outputs a sampled, shared lower-level policy for all agents $i$, from which random actions $u_t^{i,N}$ are sampled conditional on local agent states $x_t^{i,N}$.}
    \label{fig:overview-impl}
\end{figure}

As can be seen in Figure~\ref{fig:performance} for $\mu_0$ given such that the probability of access to both queues is $0.6$ and otherwise uniformly random, the mean field solution reaches its mean field performance in the $N$-agent case as $N$ grows large. This validates our theoretical findings empirically. Our solution further appears to outperform NA and PS for sufficiently many agents, as IL approaches increasingly fail to learn due to the credit assignment problem. 

Moreover, our best learned policy is close to JSQ and competitive with slight irregularities at $b_0 = 0$. Observe in Figure~\ref{fig:performance} that the MF policy gives an interpretable solution. As a queue becomes more filled, the optimal solution will be more likely to avoid assignment of packets to that queue.

\begin{figure}
	\centering
	\includegraphics[width=\linewidth]{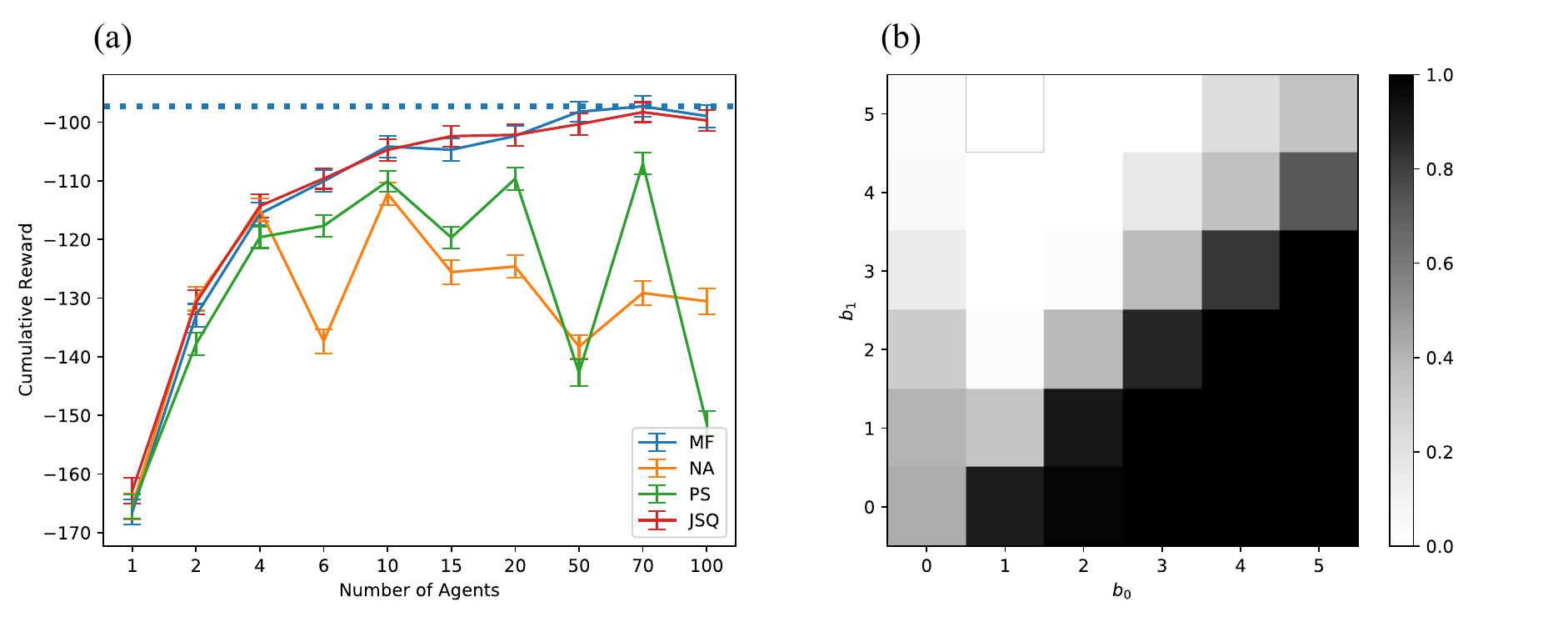}
	\caption{(a): Cumulative reward average over 500 runs with 95\% confidence interval achieved against number of agents $N$. The dotted line indicates cumulative reward of MF in the MFC MDP. NA and PS are trained separately for each $N$, while the MF policy is trained only once and used for all $N$. As $N$ grows, the MF policy performance becomes increasingly close to the MFC MDP and competitive with JSQ, while NA and PS begin to fail learning due to the credit assignment problem. (b): MF policy probabilities of assigning to queue $1$ against buffer fillings $b_0$, $b_1$ for agents with access to both queues, averaged over 500 samples.}
	\label{fig:performance}
\end{figure}

\section{CONCLUSION}
In this work, we have formulated a discrete-time mean field control model with common environment states motivated by a scheduling problem. We have rigorously shown approximate optimality as $N \to \infty$ and applied reinforcement learning to solve the MFC MDP. Empirically, we obtain competitive results for sufficiently many agents and validate our theoretical results. For future work, it could be interesting to consider partial observability of the system for schedulers, or methods to scale to large numbers of queues. Potential extensions are manifold and include dynamic agent states, major-minor systems, partial observability and general non-finite spaces.






\section*{ACKNOWLEDGMENT}
This work has been co-funded by the LOEWE initiative (Hesse, Germany) within the emergenCITY center, the European Research Council (ERC) within the Consolidator Grant CONSYN (grant agreement no. 773196) and the German Research Foundation (DFG) as part of sub-project C3 within the Collaborative Research Center (CRC) 1053 – MAKI.


\bibliographystyle{IEEEtran}
\bibliography{IEEEabrv,references}

\end{document}